\documentclass[letterpaper, 10 pt, conference]{ieeeconf}
\IEEEoverridecommandlockouts
\usepackage{times}
\usepackage{epsfig}
\usepackage{graphicx}
\usepackage{amsmath}
\usepackage{amssymb}

\usepackage{amsthm}
\usepackage{multicol}
\usepackage{multirow}
\usepackage{subcaption}
\usepackage{xfrac}

\DeclareMathOperator{\sgn}{sgn}
\newtheorem{claim}{Claim}

\title{\LARGE \bf
Learning an Uncertainty-Aware Object Detector for\\ Autonomous Driving
}

\author{Gregory P. Meyer$^{1}$ and Niranjan Thakurdesai$^{1,2}$%
\thanks{$^{1}$Uber Advanced Technologies Group}%
\thanks{$^{2}$Georgia Institute of Technology}%
\thanks{Correspondence to {\tt\small gmeyer@uber.com}}%
}

\begin{document}

\maketitle
\thispagestyle{empty}
\pagestyle{empty}

\begin{abstract}
The capability to detect objects is a core part of autonomous driving.
Due to sensor noise and incomplete data, perfectly detecting and localizing every object is infeasible.
Therefore, it is important for a detector to provide the amount of uncertainty in each prediction.
Providing the autonomous system with reliable uncertainties enables the vehicle to react differently based on the level of uncertainty.
Previous work has estimated the uncertainty in a detection by predicting a probability distribution over object bounding boxes.
In this work, we propose a method to improve the ability to learn the probability distribution by considering the potential noise in the ground-truth labeled data.
Our proposed approach improves not only the accuracy of the learned distribution but also the object detection performance.
\end{abstract}

\section{Introduction}
\label{sec:intro}
A crucial component of autonomous driving is the ability to detect and localize the surrounding objects.
To accomplish this task, autonomous vehicles are equipped with various sensors including cameras and LiDARs.
A wealth of deep learning based approaches have been proposed to perform 3D object detection using these sensors~\cite{liVehicleDetection3D2016a, chenMultiview3DObject2017, zhouVoxelNetEndtoEndLearning2018, kuJoint3DProposal2018, qiFrustumPointNets3D2018, beltranBirdNet3DObject2018, yangPIXORRealtime3D2018, liangDeepContinuousFusion2018, xuPointFusionDeepSensor2018, lasernet, lasernet++}.
Given the limited sensory information, it is unrealistic to expect any detector to flawlessly classify and localize every actor in all situations.
Therefore, it is important for a detector to provide to the autonomous system its level of uncertainty in its predictions and the uncertainties need to be reliable.

Previously proposed detectors provide the uncertainty for the object classification by predicting a categorical distribution over the classes of objects, but the majority of the previous work does not provide an uncertainty for the localization of the object.
Following~\cite{kendalluncertainities}, a few methods have been proposed to estimate the localization uncertainty by learning a probability distribution over object bounding boxes given the sensor data~\cite{lasernet, uncertainty_3d_detection_itsc, uncertainty_3d_detection}. 
These methods learn the parameters of the probability distribution by maximizing the likelihood of a ground-truth label.
In this work, we show that maximizing the likelihood, or equivalently minimizing the negative log likelihood, has undesirable properties for learning uncertainties, which could result in numerical instability and overfitting.
This is due to the likelihood being maximized when the distribution becomes a Dirac delta function with zero uncertainty and infinity probability density for a label.

Instead of assuming the labels are samples from a distribution and training the model to maximize their likelihood, we assume each label is itself a distribution, and we train the model by minimizing the Kullback-Leibler (KL) divergence between the predicted distribution and the label distribution.
We demonstrate that the KL divergence resolves the issues with the negative log likelihood and improves the performance of the model.

Assuming each label is a distribution implies the labels contain some amount of uncertainty, which is likely to be the case with human annotation.
As a result, we need a way to determine the amount of noise in each label, which is a non-trivial problem.
In this paper, we propose a heuristic to estimate the uncertainty in a label by considering both the data and how it is annotated.

With our proposed approach, we are able to improve the performance of a state-of-the-art object detector~\cite{lasernet} by only modifying the loss function used during training.
Our approach significantly improves the performance of rare objects, which we believe is the result of the KL divergence being less prone to overfitting potentially noisy labels.
Furthermore, we see an improvement in the accuracy of the predicted distribution.

In the following sections, we discuss the previous work related to estimating uncertainty with deep neural networks (Section~\ref{sec:related}), review an approach for learning the uncertainty in a detection and propose a modification that assumes the ground-truth labels are noisy (Section~\ref{sec:learning_uncertainty}), propose a way to approximate the uncertainty in a label (Section~\ref{sec:label_uncertainty}), and present experimental results for our proposed method on a large-scale autonomous driving dataset (Section~\ref{sec:experiments}).

\section{Related Work}
\label{sec:related}

\subsection{Predicting Uncertainty with Neural Networks}
Neural networks tend to make over-confident predictions and do not provide reliable estimates of uncertainty in their predictions until recently.
Bayesian modeling provides a theoretically grounded and practical framework for representing uncertainty in neural networks~\cite{mackayPracticalBayesianFramework1992, gal2016uncertainty}.
Applied to computer vision tasks, Kendall and Gal~\cite{kendalluncertainities} divided uncertainties in the Bayesian framework into two types, aleatoric and epistemic.
Aleatoric or data uncertainty is due to noise inherent in the data, so it cannot be reduced by increasing the size of the training set.
It arises from sensor noise, incomplete data, class ambiguity, and label noise~\cite{kendalluncertainities, malinin2018predictive}.
Aleatoric uncertainty is modeled by making the outputs of a neural network probabilistic, i.e. predicting a probability distribution instead of a point estimate.
Epistemic or model uncertainty is due to uncertainty in the model parameters.
It captures our ignorance about the model most suitable to explain our training data~\cite{gal2016uncertainty}.
High epistemic uncertainty means there may be another model which explains the data better.
It is important to model in the case of limited training data, but it can be reduced by collecting more training data~\cite{kendalluncertainities}.
Malinin and Gales~\cite{malinin2018predictive} modeled a third type of uncertainty called distributional uncertainty arising from a mismatch between the training and test distributions.
In this case, the model is unfamiliar with an example if it is out-of-distribution and hence, cannot make confident predictions.
They proposed a new framework called Prior Networks to explicitly predict distributional uncertainty and separate it from data uncertainty.
In this work, we focus on modeling data uncertainty since it cannot be reduced with a larger dataset.
It is important for an autonomous vehicle to understand the uncertainty in its detections due to limited sensor data so that it can plan accordingly, e.g. slow down the vehicle to collect more data.

The methods for predicting uncertainty in neural networks can be broadly divided into two categories, sampling-based and sampling-free.
One class of sampling-based approaches uses variational inference over the neural network weights~\cite{gravesPracticalVariationalInference2011a, blundellWeightUncertaintyNeural2015a}.
This involves approximating the posterior distribution of the network weights given the dataset, which is difficult to evaluate, with a more tractable distribution~\cite{gravesPracticalVariationalInference2011a, kendalluncertainities}.
Epistemic uncertainty is estimated by sampling this approximate distribution in order to draw a set of weights; each sample will result in a set of predictions which are combined to compute the epistemic uncertainty.
Gal and Ghahramani~\cite{galDropoutBayesianApproximation2016, galBayesianConvolutionalNeural2015} introduced a technique for modeling epistemic uncertainty called Monte Carlo dropout which uses multiple stochastic forward passes of the neural network with dropout~\cite{srivastava2014dropout}.
Dropout is a popular technique used during training as a form of regularization; however, with Monte Carlo dropout, it is used during inference as well to estimate epistemic uncertainty.
All of these methods rely on sampling the output several times which is computationally expensive; as a result, they are unsuitable for real-time autonomous systems.
On the other hand, sampling-free methods such as~\cite{kendalluncertainities, choiUncertaintyAwareLearningDemonstration2018} are computationally efficient and able to predict uncertainties in a single forward pass.

\subsection{Uncertainty Estimation in Object Detection}
\label{sec:uncertainty_od}
A variety of methods have been proposed for 3D object detection in the context of autonomous driving~\cite{liVehicleDetection3D2016a, chenMultiview3DObject2017, zhouVoxelNetEndtoEndLearning2018, kuJoint3DProposal2018, qiFrustumPointNets3D2018, beltranBirdNet3DObject2018, yangPIXORRealtime3D2018, liangDeepContinuousFusion2018, xuPointFusionDeepSensor2018, lasernet, lasernet++}.
Techniques for predicting localization uncertainty in object detection have been studied recently.
Jiang \textit{et al.}~\cite{uncertainty_2d_detection} predicted the intersection-over-union (IoU) between the 2D ground-truth and predicted bounding boxes as a measure of uncertainty.
Along the lines of~\cite{kendalluncertainities},~\cite{uncertainty_3d_detection_itsc} and~\cite{uncertainty_3d_detection} captured the epistemic and aleatoric uncertainties in 3D object detection using Monte Carlo dropout~\cite{galDropoutBayesianApproximation2016} and MAP inference respectively.
Furthermore, Meyer \textit{et al.}~\cite{lasernet} estimated aleatoric uncertainty in 3D object detection by predicting a multimodal distribution over bounding boxes.
In this work, we modify the way~\cite{lasernet} learns the probability distribution, resulting in better uncertainty estimates and detection performance.

\section{Proposed Method}
\label{sec:method}

In this work, we improve upon LaserNet~\cite{lasernet} which is a LiDAR-based probabilistic object detector.
LaserNet estimates the uncertainty in its detections by predicting a probability distribution over bounding boxes for each object.
In the following sections, we will review how the previous work learns a probability distribution and propose an alternative approach.

\subsection{Learning Uncertainty}
\label{sec:learning_uncertainty}

\begin{figure*}
  \centering
  \begin{subfigure}{0.32\textwidth}
    \centering
    \includegraphics[width=\linewidth]{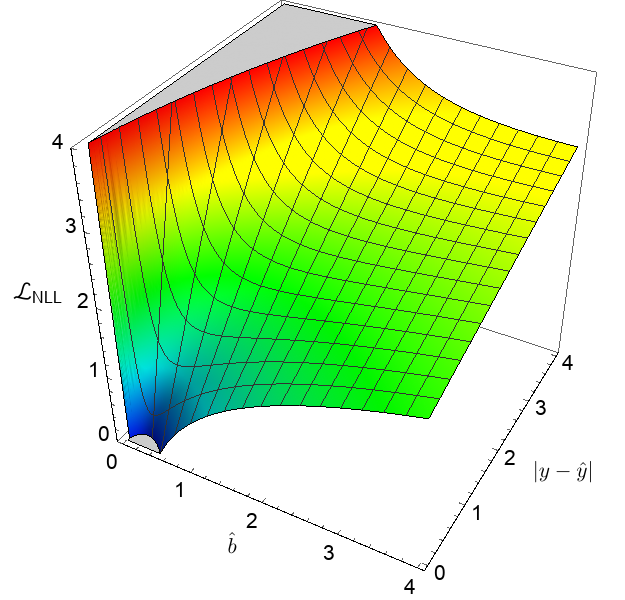}
    \caption{Negative Log Likelihood}
  \end{subfigure}
  \begin{subfigure}{0.32\textwidth}
    \centering
    \includegraphics[width=\linewidth]{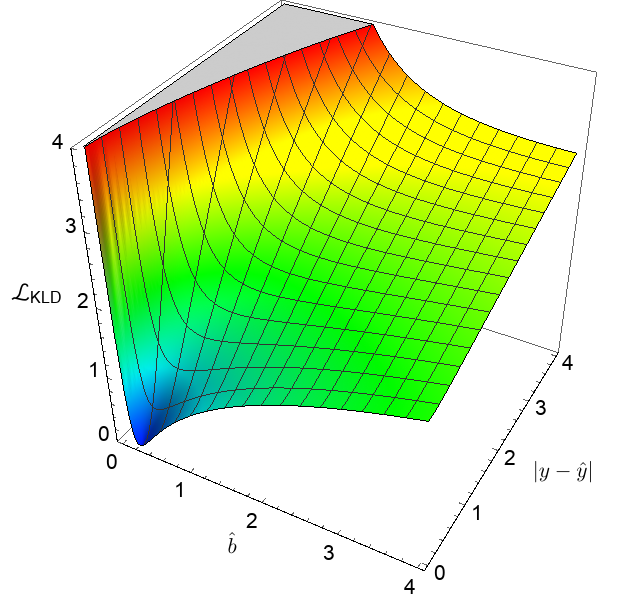}
    \caption{KL Divergence ($b=0.2$)}
  \end{subfigure}
  \begin{subfigure}{0.32\textwidth}
    \centering
    \includegraphics[width=\linewidth]{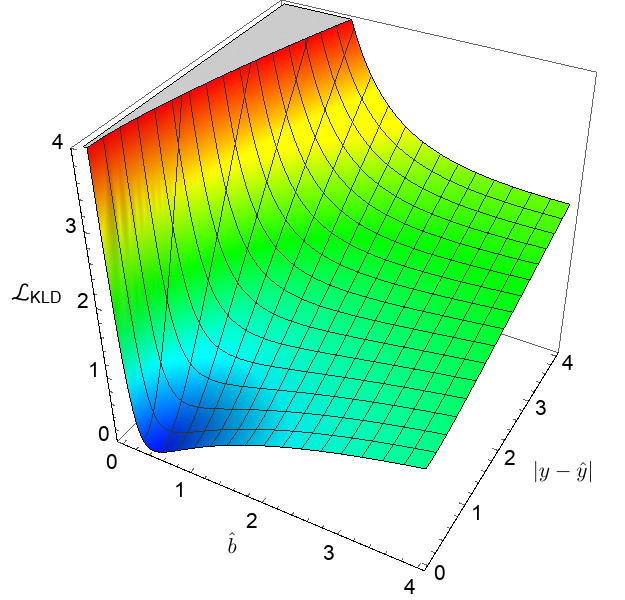}
    \caption{KL Divergence ($b=0.4$)}
  \end{subfigure}
  \caption{The negative log likelihood of a Laplace distribution and the KL divergence of Laplace distributions as a function of the prediction error, $|y - \hat{y}|$, and uncertainty, $\hat{b}$. When the error is large, the functions behave similarly. However, near their minimum, the negative log likelihood goes sharply to negative infinity, whereas the KL divergence goes to zero. The shape of the KL divergence is controlled by the label uncertainty, $b$.}
  \label{fig:losses}
\end{figure*}

To estimate the uncertainty in a detection, Meyer \textit{et al.}~\cite{lasernet} model the distribution of bounding box corners.
In~\cite{lasernet}, the corners are represented as a $D$-dimensional vector where $D$ is the number of corners multiplied by the dimensionality of each corner.
The dimensions are assumed to be drawn from independent univariate Laplace distributions; as a result, the probability density of a bounding box is defined as follows:
\begin{equation}
  \prod_{k=1}^D p(y_k|\hat{y}_k, \hat{b}_k) = \prod_{k=1}^D \frac{1}{2\hat{b}_k} \exp\left(-\frac{\left|y_k - \hat{y}_k\right|}{\hat{b}_k}\right)
\end{equation}
where $\hat{y}_k \in \mathbb{R}$ and $\hat{b}_k \in \mathbb{R}^+$ specify the mean and scale of the distribution along each dimension, respectively.
Given a LiDAR sweep denoted by $x$, LaserNet is trained to predict a distribution for each of the $M$ objects in the sweep,
\begin{equation}
  \left\{ \{\hat{y}_{1k}, \hat{b}_{1k}\}_{k=1}^D, \cdots, \{\hat{y}_{Mk}, \hat{b}_{Mk}\}_{k=1}^D \right\} = F_\theta(x)
\end{equation}
where $F_\theta$ represents a convolutional neural network parameterized by $\theta$, as well as, all pre- and post-processing~\cite{lasernet}. 
Following~\cite{kendalluncertainities}, the network is trained by minimizing the negative log likelihood over the training set,
\begin{equation}
  \hat{\theta} = \arg\min_\theta \sum_{i=1}^N \sum_{j=1}^M \sum_{k=1}^D \mathcal{L}_{NLL}(y_{ijk}, \hat{y}_{ijk}, \hat{b}_{ijk})
\end{equation}
where 
\begin{equation}
  \mathcal{L}_{NLL}(y, \hat{y}, \hat{b}) = \log2\hat{b} + \frac{\left|y - \hat{y}\right|}{\hat{b}}
  \label{eqn:nll}
\end{equation}
is the negative log likelihood of a Laplace distribution, $y_{ijk}$ is the ground-truth, $\hat{y}_{ijk}$ is the predicted mean, and $\hat{b}_{ijk}$ is predicted uncertainty for the $k$th dimension of the $j$th object bounding box in the $i$th LiDAR sweep.
The model is also trained to predict class probabilities for each object which we omit for brevity.
The partial derivatives of the negative log likelihood with respect to the predictions are
\begin{equation}
  \frac{\partial \mathcal{L}_{NLL}}{\partial \hat{y}} = -\frac{\sgn(y - \hat{y})}{\hat{b}}
\end{equation}
and
\begin{equation}
  \frac{\partial \mathcal{L}_{NLL}}{\partial \hat{b}} = \frac{1}{\hat{b}} \left(1 - \frac{\left|y - \hat{y}\right|}{\hat{b}}\right)\text{.}
\end{equation}
A problem with minimizing the negative log likelihood is that as $\left|y - \hat{y}\right| \rightarrow 0$,
\begin{equation}
  \frac{\partial \mathcal{L}_{NLL}}{\partial \hat{b}} \rightarrow \frac{1}{\hat{b}}\text{,}
\end{equation}
and the derivatives for both $\hat{y}$ and $\hat{b}$ explode as $\hat{b} \rightarrow 0$.
In practice, gradient clipping can be employed to prevent the model from diverging, but the model is still prone to overfitting since the magnitude of the gradient does not tend to zero as the prediction error goes to zero, $|y-\hat{y}| \rightarrow 0$.

The issue with the previous approach is that it assumes the labels are free of noise; therefore, \eqref{eqn:nll} is minimized only when the prediction has no error and uncertainty, i.e. $|y-\hat{y}|=0$ and $\hat{b}=0$.
Instead, let us assume each label contains some amount of noise and is itself a Laplace distribution.
As a result, we have the following distributions for each dimension of the bounding box:
\begin{equation}
  p(y_k^* | y_k, b_k) = \frac{1}{2b_k} \exp\left(-\frac{\left|y_k^* - y_k\right|}{b_k}\right)
\end{equation}
and
\begin{equation}
  q(y_k^* | \hat{y}_k, \hat{b}_k) = \frac{1}{2\hat{b}_k} \exp\left(-\frac{\left|y_k^* - \hat{y}_k\right|}{\hat{b}_k}\right)
\end{equation}
where $y_k^*$ is the unknown true position of the bounding box, $p(y_k^* | y_k, b_k)$ represents our uncertainty in the label, and $q(y_k^* | \hat{y}_k, \hat{b}_k)$ represents our uncertainty in the prediction.
Assuming we know the distribution of each label, then we can train the model by minimizing the Kullback-Leibler (KL) divergence between each of the label distributions and the predicted distributions,
\begin{equation}
  \hat{\theta} = \arg\min_\theta \sum_{i=1}^N \sum_{j=1}^M \sum_{k=1}^D \mathcal{L}_{KLD}(y_{ijk}, \hat{y}_{ijk}, b_{ijk}, \hat{b}_{ijk}) \\
\end{equation}
where
\begin{equation}
  \mathcal{L}_{KLD}(y, \hat{y}, b, \hat{b}) = \log\frac{\hat{b}}{b} + \frac{b \exp\left(-\frac{\left|y - \hat{y}\right|}{b} \right) + \left|y - \hat{y}\right|}{\hat{b}} - 1
  \label{eqn:kld}
\end{equation}
is the KL divergence of two Laplace distributions \cite{huber_loss}.
The partial derivatives of the KL divergence with respect to the predictions are
\begin{equation}
  \frac{\partial \mathcal{L}_{KLD}}{\partial \hat{y}} = -\frac{\sgn(y - \hat{y})}{\hat{b}} \left(1 - \exp\left(-\frac{\left|y - \hat{y}\right|}{b} \right) \right)
  \label{eqn:kld_y}
\end{equation}
and
\begin{equation}
  \frac{\partial \mathcal{L}_{KLD}}{\partial \hat{b}} = \frac{1}{\hat{b}} \left(1 - \frac{b \exp\left(-\frac{\left|y - \hat{y}\right|}{b} \right) + \left|y - \hat{y}\right|}{\hat{b}}\right)\text{.}
  \label{eqn:kld_b}
\end{equation}
As $\left|y - \hat{y}\right| \rightarrow 0$ and $b > 0$,
\begin{equation}
  \frac{\partial \mathcal{L}_{KLD}}{\partial \hat{y}} \rightarrow 0
\end{equation}
and
\begin{equation}
  \frac{\partial \mathcal{L}_{KLD}}{\partial \hat{b}} \rightarrow \frac{1}{\hat{b}} \left(1 - \frac{b}{\hat{b}}\right)
\end{equation}
since the exponential terms in \eqref{eqn:kld_y} and \eqref{eqn:kld_b} go to one.
As $\left|y - \hat{y}\right| \rightarrow 0$ and $\hat{b} \rightarrow b$, the derivatives for both $\hat{y}$ and $\hat{b}$ go to zero, which is a desirable property for a loss function.
However, if the noise in the label is set to zero, $b=0$, the exponential terms in \eqref{eqn:kld_y} and \eqref{eqn:kld_b} become zero regardless of the prediction error, and the derivatives of the KL divergence become identical to those of the negative log likelihood.

A comparison of the loss functions is shown in Fig.~\ref{fig:losses}.
The shape of the loss surface is similar for both the negative log likelihood and the KL divergence when the prediction error is large since the exponential terms in \eqref{eqn:kld_y} and \eqref{eqn:kld_b} go to zero as the error increases.
However, the behavior of the loss functions near the minimum differs significantly.
The negative log likelihood goes to negative infinity when the prediction error and uncertainty goes to zero, whereas the KL divergence becomes zero when the predicted distribution matches the label distribution.
Moreover, with the KL divergence, increasing the uncertainty of a label reduces the loss for that label as long as the prediction uncertainty does not underestimate the label uncertainty.
Refer to the Appendix for a proof.
This property of the KL divergence can be leveraged to prevent the model from overfitting to noisy labels.
Next, we discuss how we approximate the uncertainty of a label.

\begin{figure*}
  \centering
  \begin{subfigure}{0.3\textwidth}
    \centering
    \includegraphics[width=\linewidth]{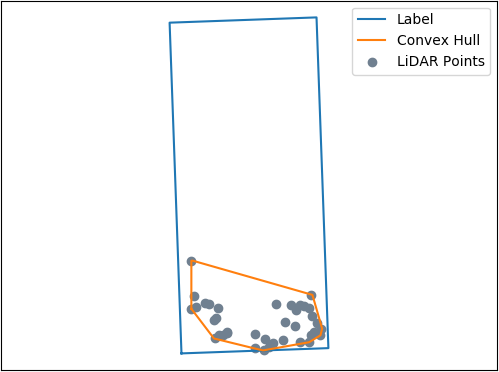}
    \caption{0.16 IoU}
  \end{subfigure}
  \begin{subfigure}{0.3\textwidth}
    \centering
    \includegraphics[width=\linewidth]{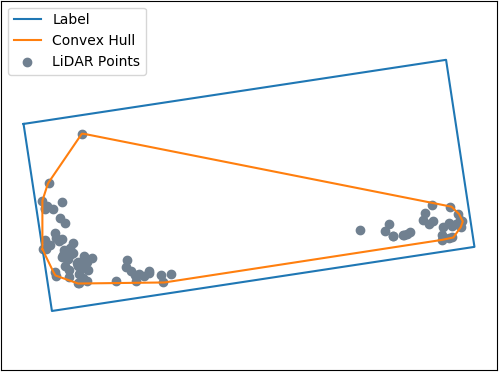}
    \caption{0.49 IoU}
  \end{subfigure}
  \begin{subfigure}{0.3\textwidth}
    \centering
    \includegraphics[width=\linewidth]{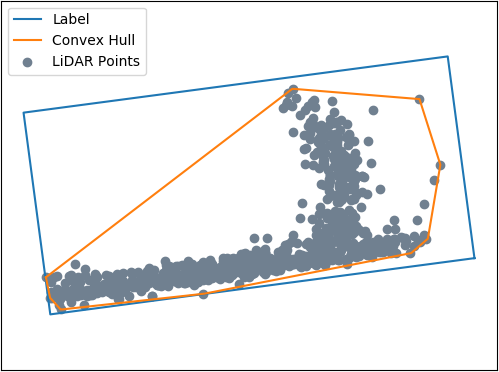}
    \caption{0.59 IoU}
  \end{subfigure} \\ \vspace{0.75em}
  \begin{subfigure}{0.3\textwidth}
    \centering
    \includegraphics[width=\linewidth]{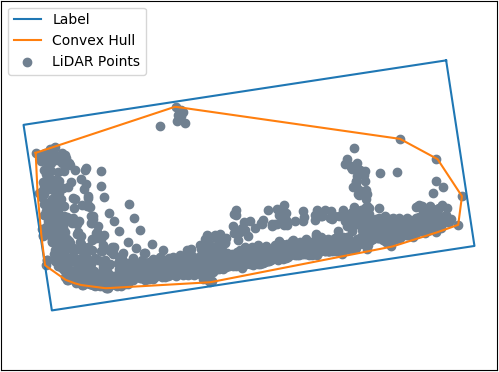}
    \caption{0.72 IoU}
  \end{subfigure}
  \begin{subfigure}{0.3\textwidth}
    \centering
    \includegraphics[width=\linewidth]{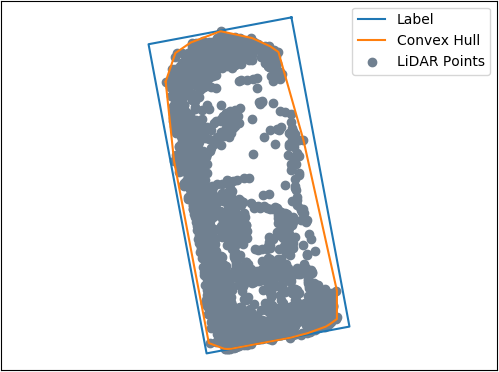}
    \caption{0.85 IoU}
  \end{subfigure}
  \begin{subfigure}{0.3\textwidth}
    \centering
    \includegraphics[width=\linewidth]{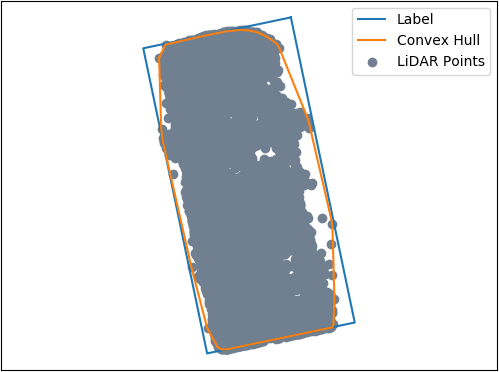}
    \caption{0.90 IoU}
  \end{subfigure}
  \caption{Examples of labeled vehicles that motivate the use of the IoU between the label bounding box and the convex hull of aggregated LiDAR points to estimate the noise in the annotation. We speculate that there is an inverse non-linear relationship between the uncertainty of a label and the IoU between the label and the convex hull.}
  \label{fig:ious}
\end{figure*}

\subsection{Estimating Label Uncertainty}
\label{sec:label_uncertainty}

To utilize KL divergence, we need to identify the distribution of every label.
One approach could be to have an annotator provide the distribution as he/she labels the data.
However, directly annotating the uncertainty of the label may be challenging and time-consuming.
Another possibility is to draw samples from the underlying label distribution by having multiple annotators provide labels for the same data, and then estimate the distribution from the samples.
Unfortunately, this process would be expensive and there is no guarantee that the label samples are independent and identically distributed (i.i.d.).
In this work, we explore ways to approximate the uncertainty in the labeled data.

The simplest approximation would be to assume a Laplace distribution with a constant uncertainty for every label, where the label itself defines the mean $y_{ijk}$ and a single scale $b_{ijk}$ is shared across all labels.
Depending on the scale we choose, we may significantly underestimate or overestimate the noise in a label.
To obtain a better approximation of the label uncertainty, we need to consider the data and how it is annotated.
In the case of object detection with LiDAR, annotators label objects within each LiDAR sweep and typically have access to previous and future sweeps.
For the dataset used in our experiments (Section~\ref{sec:experiments}), each LiDAR sweep is transformed from the sensor's coordinate system into a global coordinate frame which accounts for the ego-motion of the autonomous vehicle.
Labels are provided as oriented rectangles in the x-y plane (bird's eye view) of the global coordinate frame.
To estimate the noise in a label, we start by accumulating all the LiDAR points inside the label across all the sweeps where the label is visible.
To account for motion of the object across sweeps, we assume the object is rigid and compute its translation and rotation from the current sweep to an arbitrary reference sweep.
For example, let $p_{ij}$ be a point inside the $i$th label and captured during the $j$th sweep.
To transform $p_{ij}$ from the $j$th to the $k$th sweep, we perform the following operation:
\begin{equation}
  p^k_{ij} = \mathbf{R}_z(\theta_{ik} - \theta_{ij}) \left[p_{ij} - c_{ij}\right] + c_{ik}
\end{equation}
where $c_{ij}$ and $c_{ik}$ are the centers of the label in the $j$th and $k$th sweeps, $\theta_{ij}$ and $\theta_{ik}$ are the orientations of the label in their respective sweeps, and $\mathbf{R}_z(\theta)$ is a rotation matrix about the $z$-axis parameterized by $\theta$.
Afterwards, we determine the convex hull of the accumulated points and compute the intersection-over-union (IoU) between the convex hull and the label.
We speculate that the higher the IoU, the smaller the ambiguity in the label, as illustrated in Fig.~\ref{fig:ious}.
For each label, we use an exponential function of the form
\begin{equation}
  y = \alpha \exp(-\beta \cdot x) + \gamma
  \label{eqn:mapping}
\end{equation}
to map from the IoU between the label and the convex hull to the uncertainty of the label, where $\alpha$, $\beta$, and $\gamma$ are hyper-parameters to control the mapping.
In the next section, we evaluate the effect of KL divergence on the predicted distribution where our heuristic is used to estimate the scale of the label distribution.

\begin{table*}[t]
  \centering
  \caption{Loss Function and Label Uncertainty Experimental Results}
    \begin{tabular}{c|ccc||ccc|c}
      \hline
      \multirow{2}{*}{Loss Function} & \multicolumn{3}{c||}{Label Uncertainty (m)} & \multicolumn{3}{c|}{Average Precision (\%)} & Mean Average \\
      & Vehicle & Bike & Pedestrian & Vehicle & Bike & Pedestrian & Precision (\%) \\
      \hline
      $\mathcal{L}_{NLL}$ & 0.00 & 0.00 & 0.00 & 85.34 & 61.93 & 80.37 & 75.88 \\
      \hline
      $\mathcal{L}_{KLD}$ & 0.50 & 0.50 & 0.50 & 69.64 & 55.35 & 76.73 & 67.24 \\
      $\mathcal{L}_{KLD}$ & 0.25 & 0.25 & 0.25 & 80.23 & 60.10 & 80.66 & 73.66 \\
      $\mathcal{L}_{KLD}$ & 0.10 & 0.10 & 0.10 & 85.20 & 63.32 & 82.08 & 76.87 \\
      $\mathcal{L}_{KLD}$ & 0.05 & 0.05 & 0.05 & 85.91 & 62.81 & 81.90 & 76.87 \\
      $\mathcal{L}_{KLD}$ & 0.01 & 0.01 & 0.01 & 85.48 & 63.78 & 82.37 & 77.21 \\
      \hline
      $\mathcal{L}_{KLD}$ & $1.00 \rightarrow 0.05 \rightarrow 0.01$ & $1.00 \rightarrow 0.05 \rightarrow 0.01$ & $1.00 \rightarrow 0.05 \rightarrow 0.01$ & \textbf{86.05} & 63.71 & 82.02 & 77.26 \\
      $\mathcal{L}_{KLD}$ & $0.50 \rightarrow 0.05 \rightarrow 0.01$ & $0.25 \rightarrow 0.05 \rightarrow 0.01$ & $0.10 \rightarrow 0.05 \rightarrow 0.01$ & 86.04 & 64.29 & 82.43 & 77.59 \\
      $\mathcal{L}_{KLD}$ & $1.00 \rightarrow 0.05 \rightarrow 0.01$ & $0.50 \rightarrow 0.05 \rightarrow 0.01$ & $0.25 \rightarrow 0.05 \rightarrow 0.01$ & 85.87 & 64.58 & 82.41 & 77.62 \\
      $\mathcal{L}_{KLD}$ & $2.00 \rightarrow 0.05 \rightarrow 0.01$ & $1.00 \rightarrow 0.05 \rightarrow 0.01$ & $0.50 \rightarrow 0.05 \rightarrow 0.01$ & 85.74 & \textbf{64.82} & \textbf{82.61} & \textbf{77.72} \\
      \hline
    \end{tabular}
  \label{tab:ablation}
\end{table*}

\section{Experiments}
\label{sec:experiments}

Our proposed method is evaluated on the ATG4D dataset which contains 5,000 sequences for training and 500 for validation.
The training sequences are sampled at 10 Hz, and the validation sequences are sampled at 0.5 Hz.
The training set contains 1.2 million sweeps, and the validation set has 5,969 sweeps.
A Velodyne 64E LiDAR was used to capture all of the sweeps in the dataset.

For all of our experiments, we use LaserNet~\cite{lasernet} as our base detector and only modify the loss function used to learn the probability distribution over bounding boxes (Equation (9) in~\cite{lasernet}).
We change the loss function from the negative log likelihood of a Laplace distribution to the KL divergence of Laplace distributions.
All other settings remain unchanged; please refer to~\cite{lasernet} for details.

We experimented with two ways to approximate the label uncertainty: a constant uncertainty for all labels and a heuristically obtained label uncertainty.
In either case, the label distribution is assumed to be a Laplace distribution, and the mean of the distribution, $y_{ijk}$, is defined by the label.
When using a fixed uncertainty, a single scale parameter, $b_{ijk}$, is used for all labels.
We performed a variety of experiments with different values for the scale.
When utilizing the heuristic, we estimate the noise in each label.
For each label in the training set, we compute the IoU between the label and the convex hull of the LiDAR points observed within the label using the approach described in Section~\ref{sec:label_uncertainty}.
A histogram of IoU values is depicted in Fig.~\ref{fig:histogram} and shows there is a diverse set of labeled data in the dataset.
Equation \eqref{eqn:mapping} is used to map from the IoU to $b_{ijk}$ for each label, and we experimented with different values for $\alpha$, $\beta$, and $\gamma$.

\begin{figure}[t]
    \centering
    \includegraphics[width=0.4\textwidth]{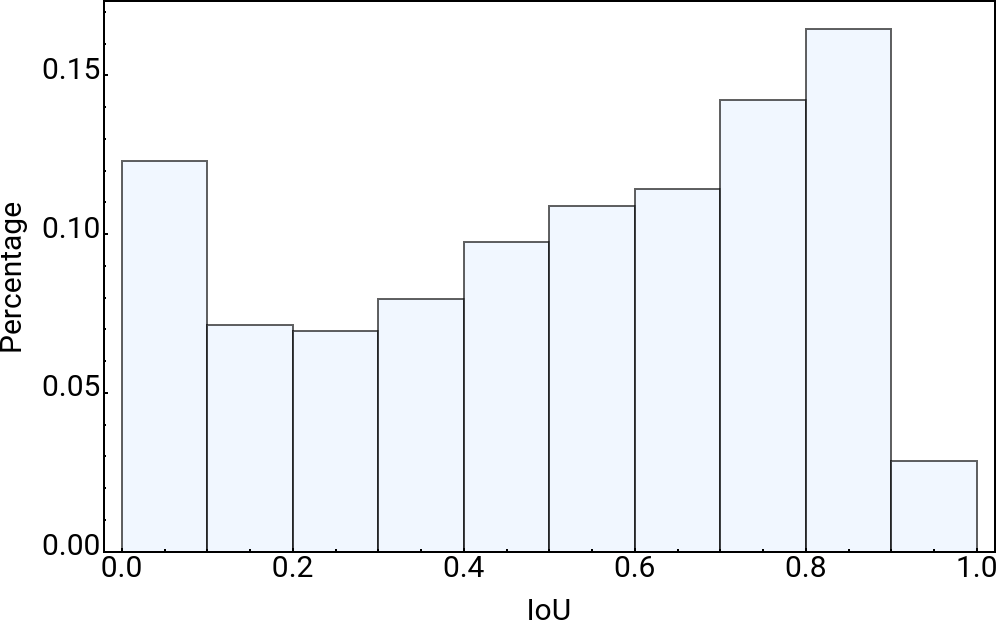}
    \hspace{0.5em}
    \caption{The distribution of IoU values in the ATG4D dataset.}
    \label{fig:histogram}
\end{figure}

\subsection{Detection Evaluation}
Following the previous work, we evaluate detections within the front $90^\circ$ field of view of the LiDAR and up to 70 meters away from the sensor.
The average precision (AP) metric is used to measure detection performance.
To be considered a true positive, a vehicle detection needs to achieve an IoU of 0.7 with a ground-truth bounding box, and for bike and pedestrian detections, an IoU of 0.5 is required.

Detection performance for the various experiments are presented in Table~\ref{tab:ablation}.
The values of $b_{ijk}$ for each experiment are listed in the table under ``Label Uncertainty.''
When a constant uncertainty is used for all labels, only a single value is specified for each class.
When using the heuristic to estimate the uncertainty for each label, three values are shown with the notation $b_{0.0} \rightarrow b_{0.5} \rightarrow b_{1.0}$ which correspond to the value of $b_{ijk}$ when the IoU between the label and the convex hull is zero, one-half, and one, respectively.
These values are used to determine the parameters $\alpha$, $\beta$, and $\gamma$.

By replacing the negative log likelihood with the KL divergence, we obtain an improvement in performance, as long as, the label uncertainty is not overestimated (row two and three in Table~\ref{tab:ablation}).
By approximating the label uncertainty with our proposed heuristic, we see a considerable improvement in the rarest class, the bike class.
We believe the gain in performance is due to our method reducing the amount of overfitting to potentially noisy labels, specifically, vehicle labels which are by far the most commonly occurring class.

A comparison between our proposed method (last row in Table~\ref{tab:ablation}) and recent state-of-the-art object detectors is shown in Table~\ref{tab:results}.
Our approach outperforms the previous methods that only utilize LiDAR data.
Furthermore, by only changing the loss function, we observe a similar gain in performance as adding an additional sensing modality (LaserNet++~\cite{lasernet++}).

\begin{table}[t]
  \centering
  \caption{Comparisons to State-of-the-Art Methods}
    \begin{tabular}{c|c||ccc}
      \hline
      \multirow{2}{*}{Method} & \multirow{2}{*}{Input} & \multicolumn{3}{c}{Average Precision (\%)} \\
      & & Vehicle & Bike & Pedestrian \\
      \hline
      Proposed Method & LiDAR & 85.74 & 64.82 & 82.61 \\
      \hline
      PIXOR~\cite{yangPIXORRealtime3D2018} & LiDAR & 80.99 & - & - \\
      PIXOR++~\cite{hdnet} & LiDAR & 82.63 & - & - \\
      ContFuse~\cite{liangDeepContinuousFusion2018} & LiDAR & 83.13 & 57.27 & 73.51 \\
      LaserNet~\cite{lasernet} & LiDAR  & 85.34 & 61.93 & 80.37 \\
      \hline
      ContFuse~\cite{liangDeepContinuousFusion2018} & LiDAR+RGB & 85.17 & 61.13 & 76.84 \\
      LaserNet++~\cite{lasernet++} & LiDAR+RGB & \textbf{86.23} & \textbf{65.68} & \textbf{83.42} \\
      \hline
    \end{tabular}
  \label{tab:results}
\end{table}

\begin{figure*}
  \begin{subfigure}{.5\textwidth}
    \centering
    \includegraphics[width=.82\linewidth]{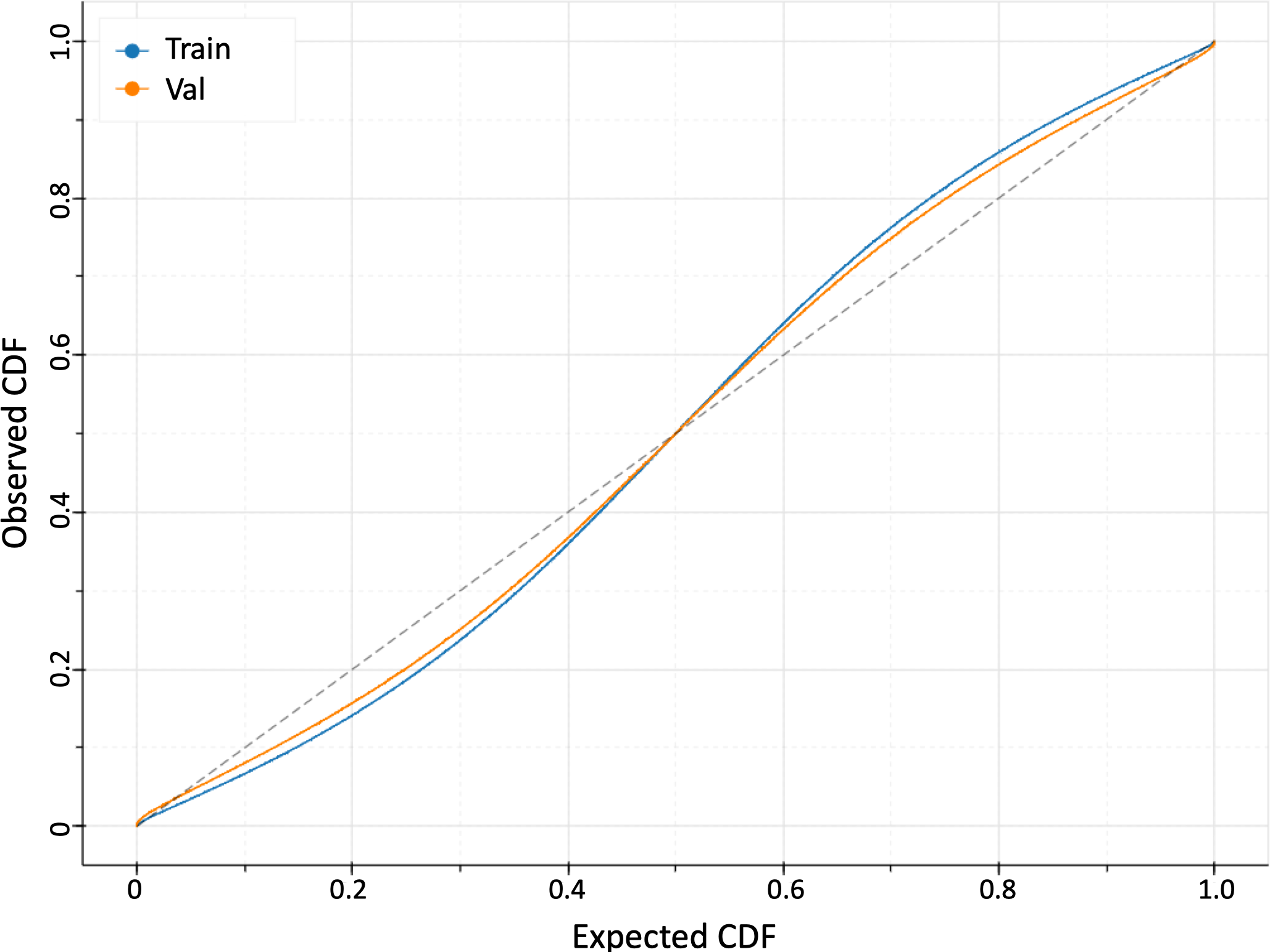}
    \caption{LaserNet~\cite{lasernet}}
  \end{subfigure}
  \begin{subfigure}{.5\textwidth}
    \centering
    \includegraphics[width=.82\linewidth]{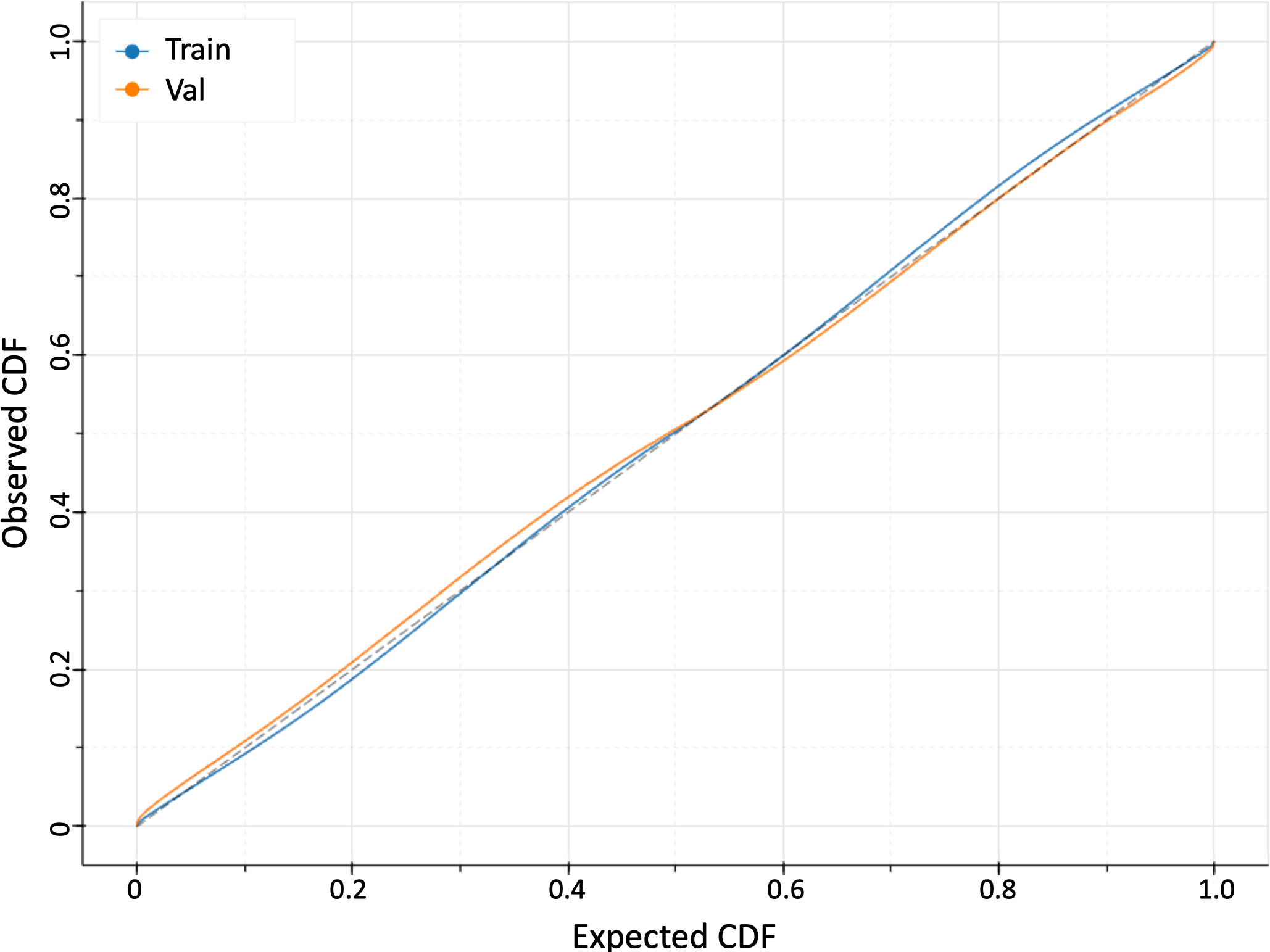}
    \caption{Proposed Method}
  \end{subfigure}
  \caption{Calibration plots showing the reliability of the predicted distribution from LaserNet~\cite{lasernet} and our proposed method. A perfectly calibrated model would follow the dashed line (the observed CDF matches the expected CDF at all probabilities).}
  \label{fig:variance}
\end{figure*}

\subsection{Uncertainty Evaluation}
To evaluate the predicted probability distribution, we compare the expected cumulative distribution function (CDF) to the observed CDF.
For each ground-truth label from a particular class, we calculate its standard score given the parameters of the predicted distribution.
A cumulative histogram is created from the standard scores, and it is compared to the CDF of a standard Laplace distribution.
The resulting plots for the vehicle class are shown in Fig.~\ref{fig:variance}.
Compared to~\cite{lasernet}, our proposed approach is capable of learning a better calibrated probability distribution.

\section{Conclusion}
\label{sec:conclusion}

In this work, we presented an approach to learn the uncertainty of a detection given the sensor data while being aware of the uncertainty in the labeled data.
Our proposed method improves a state-of-the-art probabilistic object detector~\cite{lasernet} in terms of both object detection performance and the accuracy of the learned probability distribution.
The sole difference is that~\cite{lasernet} makes the implicit assumption that labels are noise-free and uses the negative log likelihood to learn the distribution, whereas our proposed method uses a heuristic to estimate the uncertainty in a label and learns the distribution by minimizing the KL divergence.
We believe the improvement is due to the KL divergence being more well-behaved during training than the negative log likelihood.

\section*{Appendix}

In this section, we provide a proof for the claim made in Section \ref{sec:learning_uncertainty}.
\begin{claim}
  The following inequality holds for all $y, \hat{y} \in \mathbb{R}$:
  \begin{align}
    \begin{split}
      \log\frac{\hat{b}}{b_1} + \frac{b_1 \exp\left(-\frac{\left|y - \hat{y}\right|}{b_1} \right) + \left|y - \hat{y}\right|}{\hat{b}} - 1 > \\
      \log\frac{\hat{b}}{b_2} + \frac{b_2 \exp\left(-\frac{\left|y - \hat{y}\right|}{b_2} \right) + \left|y - \hat{y}\right|}{\hat{b}} - 1
    \end{split}
  \end{align}
  when $\hat{b} \geq b_2 > b_1 > 0$.
\end{claim}

\begin{proof}
  The inequality is equivalent to
  \begin{equation}
    f_1(x) = \log\frac{b_2}{b_1} + \frac{b_1 \exp\left(-\frac{1}{b_1} x \right) - b_2 \exp\left(-\frac{1}{b_2} x\right)}{\hat{b}} > 0
  \end{equation}
  where $x = \left|y - \hat{y}\right| \geq 0$.
  To prove the inequality, we will utilize the mean value theorem; therefore, we need to show that $f_1(0) > 0$ and $f'_1(x) \geq 0$ when $x \geq 0$.
  The first derivative of $f_1(x)$ is
  \begin{equation}
    f'_1(x) = \frac{1}{\hat{b}}\left(\exp\left(-\frac{1}{b_2} x\right) - \exp\left(-\frac{1}{b_1} x\right) \right)\text{.}
  \end{equation}
  The function $\exp(-x)$ monotonically decreases; therefore, $\exp(-x_1) > \exp(-x_2)$ when $x_2 > x_1$.
  Since $\sfrac{1}{b_1} > \sfrac{1}{b_2}$ and $\hat{b} > 0$, $f'_1(x) \geq 0$ when $x \geq 0$.
  Next, we need to demonstrate that
  \begin{equation}
    f_1(0) = \log\frac{b_2}{b_1} + \frac{b_1 - b_2}{\hat{b}} > 0\text{.}
  \end{equation}
  Since $\hat{b} \geq b_2 > b_1 > 0$,
  \begin{equation}
    0 > \frac{b_1 - b_2}{\hat{b}} \geq \frac{b_1 - b_2}{b_2}\text{.}
  \end{equation}
  Therefore, if
  \begin{equation}
    \log\frac{b_2}{b_1} + \frac{b_1 - b_2}{b_2} > 0
  \end{equation}
  or equivalently
  \begin{equation}
    f_2(\rho) = \log\rho + \frac{1}{\rho} - 1 > 0
  \end{equation}
  where $\rho = \sfrac{b_2}{b_1} > 1$, then $f_1(0) > 0$.
  Again leveraging the mean value theorem, we can show that $f_2(\rho) > 0$ when $\rho > 1$.
  The first derivative of $f_2(\rho)$ is
  \begin{equation}
      f'_2(\rho) = \frac{1}{\rho} - \frac{1}{\rho^2}\text{.}
  \end{equation}
  When $\rho > 1$, $\rho^2 > \rho$ and $f'_2(\rho) > 0$.
  Since $f_2(1) = 0$, $f_2(\rho) > 0$ when $\rho > 1$.
  As a result, $f_1(0) > 0$ which completes the proof.
\end{proof}

\addtolength{\textheight}{-10cm}

\bibliographystyle{IEEEtran}
\bibliography{bibliography}

\begin{thebibliography}{10}
\providecommand{\url}[1]{#1}
\csname url@rmstyle\endcsname
\providecommand{\newblock}{\relax}
\providecommand{\bibinfo}[2]{#2}
\providecommand\BIBentrySTDinterwordspacing{\spaceskip=0pt\relax}
\providecommand\BIBentryALTinterwordstretchfactor{4}
\providecommand\BIBentryALTinterwordspacing{\spaceskip=\fontdimen2\font plus
\BIBentryALTinterwordstretchfactor\fontdimen3\font minus
  \fontdimen4\font\relax}
\providecommand\BIBforeignlanguage[2]{{%
\expandafter\ifx\csname l@#1\endcsname\relax
\typeout{** WARNING: IEEEtran.bst: No hyphenation pattern has been}%
\typeout{** loaded for the language `#1'. Using the pattern for}%
\typeout{** the default language instead.}%
\else
\language=\csname l@#1\endcsname
\fi
#2}}

\bibitem{liVehicleDetection3D2016a}
B.~Li, T.~Zhang, and T.~Xia, ``Vehicle detection from 3{D} lidar using fully
  convolutional network,'' in \emph{Proceedings of Robotics: Science and
  Systems (RSS)}, 2016.

\bibitem{chenMultiview3DObject2017}
X.~Chen, H.~Ma, J.~Wan, B.~Li, and T.~Xia, ``Multi-view 3{D} object detection
  network for autonomous driving,'' in \emph{Proceedings of the IEEE Conference
  on Computer Vision and Pattern Recognition (CVPR)}, 2017.

\bibitem{zhouVoxelNetEndtoEndLearning2018}
Y.~Zhou and O.~Tuzel, ``{VoxelNet}: End-to-end learning for point cloud based
  3{D} object detection,'' in \emph{Proceedings of the IEEE Conference on
  Computer Vision and Pattern Recognition (CVPR)}, 2018.

\bibitem{kuJoint3DProposal2018}
J.~Ku, M.~Mozifian, J.~Lee, A.~Harakeh, and S.~L. Waslander, ``Joint 3{D}
  proposal generation and object detection from view aggregation,'' in
  \emph{Proceedings of the IEEE/RSJ International Conference on Intelligent
  Robots and Systems (IROS)}, 2018.

\bibitem{qiFrustumPointNets3D2018}
C.~R. Qi, W.~Liu, C.~Wu, H.~Su, and L.~J. Guibas, ``Frustum pointnets for 3{D}
  object detection from {RGB-D} data,'' in \emph{Proceedings of the IEEE
  Conference on Computer Vision and Pattern Recognition (CVPR)}, 2018.

\bibitem{beltranBirdNet3DObject2018}
J.~Beltr{\'a}n, C.~Guindel, F.~M. Moreno, D.~Cruzado, F.~Garc{\'\i}a, and
  A.~De~La~Escalera, ``{BirdNet}: A 3{D} object detection framework from
  {LiDAR} information,'' in \emph{Proceedings of the International Conference
  on Intelligent Transportation Systems (ITSC)}, 2018.

\bibitem{yangPIXORRealtime3D2018}
B.~Yang, W.~Luo, and R.~Urtasun, ``{PIXOR}: Real-time 3{D} object detection
  from point clouds,'' in \emph{Proceedings of the IEEE Conference on Computer
  Vision and Pattern Recognition (CVPR)}, 2018.

\bibitem{liangDeepContinuousFusion2018}
M.~Liang, B.~Yang, S.~Wang, and R.~Urtasun, ``Deep continuous fusion for
  multi-sensor 3{D} object detection,'' in \emph{Proceedings of the European
  Conference on Computer Vision (ECCV)}, 2018.

\bibitem{xuPointFusionDeepSensor2018}
D.~Xu, D.~Anguelov, and A.~Jain, ``Pointfusion: Deep sensor fusion for 3{D}
  bounding box estimation,'' in \emph{Proceedings of the IEEE Conference on
  Computer Vision and Pattern Recognition (CVPR)}, 2018.

\bibitem{lasernet}
G.~P. Meyer, A.~Laddha, E.~Kee, C.~Vallespi-Gonzalez, and C.~K. Wellington,
  ``{LaserNet}: An efficient probabilistic 3{D} object detector for autonomous
  driving,'' in \emph{Proceedings of the IEEE Conference on Computer Vision and
  Pattern Recognition (CVPR)}, 2019.

\bibitem{lasernet++}
G.~P. Meyer, J.~Charland, D.~Hegde, A.~Laddha, and C.~Vallespi-Gonzalez,
  ``Sensor fusion for joint 3{D} object detection and semantic segmentation,''
  in \emph{Proceedings of the IEEE Conference on Computer Vision and Pattern
  Recognition Workshops (CVPRW)}, 2019.

\bibitem{kendalluncertainities}
A.~Kendall and Y.~Gal, ``What uncertainties do we need in {Bayesian} deep
  learning for computer vision?'' in \emph{Proceedings of Advances in Neural
  Information Processing Systems (NIPS)}, 2017.

\bibitem{uncertainty_3d_detection_itsc}
D.~Feng, L.~Rosenbaum, and K.~Dietmayer, ``Towards safe autonomous driving:
  Capture uncertainty in the deep neural network for lidar 3{D} vehicle
  detection,'' in \emph{Proceedings of the International Conference on
  Intelligent Transportation Systems (ITSC)}, 2018.

\bibitem{uncertainty_3d_detection}
D.~Feng, L.~Rosenbaum, F.~Timm, and K.~Dietmayer, ``Leveraging heteroscedastic
  aleatoric uncertainties for robust real-time {LiDAR} 3{D} object detection,''
  \emph{arXiv preprint arXiv:1809.05590}, 2018.

\bibitem{mackayPracticalBayesianFramework1992}
D.~J.~C. MacKay, ``A practical {Bayesian} framework for backpropagation
  networks,'' \emph{Neural Computation}, vol.~4, no.~3, pp. 448--472, 1992.

\bibitem{gal2016uncertainty}
Y.~Gal, ``Uncertainty in deep learning,'' Ph.D. dissertation, University of
  Cambridge, 2016.

\bibitem{malinin2018predictive}
A.~Malinin and M.~Gales, ``Predictive uncertainty estimation via prior
  networks,'' in \emph{Proceedings of Advances in Neural Information Processing
  Systems (NIPS)}, 2018.

\bibitem{gravesPracticalVariationalInference2011a}
A.~Graves, ``Practical variational inference for neural networks,'' in
  \emph{Proceedings of Advances in Neural Information Processing Systems
  (NIPS)}, 2011.

\bibitem{blundellWeightUncertaintyNeural2015a}
C.~Blundell, J.~Cornebise, K.~Kavukcuoglu, and D.~Wierstra, ``Weight
  uncertainty in neural networks,'' in \emph{Proceedings of the International
  Conference on Machine Learning (ICML)}, 2015.

\bibitem{galDropoutBayesianApproximation2016}
Y.~Gal and Z.~Ghahramani, ``Dropout as a {Bayesian} approximation:
  {Representing} model uncertainty in deep learning,'' in \emph{Proceedings of
  the International Conference on Machine Learning (ICML)}, 2016.

\bibitem{galBayesianConvolutionalNeural2015}
------, ``Bayesian convolutional neural networks with {B}ernoulli approximate
  variational inference,'' in \emph{Proceedings of the International Conference
  on Learning Representations (ICLR) Workshops}, 2016.

\bibitem{srivastava2014dropout}
N.~Srivastava, G.~Hinton, A.~Krizhevsky, I.~Sutskever, and R.~Salakhutdinov,
  ``Dropout: A simple way to prevent neural networks from overfitting,''
  \emph{The {Journal} of {Machine} {Learning} {Research}}, vol.~15, no.~1, pp.
  1929--1958, 2014.

\bibitem{choiUncertaintyAwareLearningDemonstration2018}
S.~Choi, K.~Lee, S.~Lim, and S.~Oh, ``Uncertainty-aware learning from
  demonstration using mixture density networks with sampling-free variance
  modeling,'' in \emph{Proceedings of the IEEE International Conference on
  Robotics and Automation (ICRA)}, 2018.

\bibitem{uncertainty_2d_detection}
B.~Jiang, R.~Luo, J.~Mao, T.~Xiao, and Y.~Jiang, ``Acquisition of localization
  confidence for accurate object detection,'' in \emph{Proceedings of the
  European Conference on Computer Vision (ECCV)}, 2018.

\bibitem{huber_loss}
G.~P. Meyer, ``An alternative probabilistic interpretation of the {H}uber
  loss,'' \emph{arXiv preprint arXiv:1911.02088}, 2019.

\bibitem{hdnet}
B.~Yang, M.~Liang, and R.~Urtasun, ``{HDNET}: Exploiting {HD} maps for 3{D}
  object detection,'' in \emph{Proceedings of the Conference on Robot Learning
  (CoRL)}, 2018.

\end{thebibliography}

\end{document}